\documentclass[letterpaper, 10 pt, conference]{ieeeconf}

\IEEEoverridecommandlockouts
\newcommand\numberthis{\addtocounter{equation}{1}\tag{\theequation}}
\usepackage[utf8]{inputenc}

\usepackage{lipsum}

\usepackage{algorithm}
\usepackage{amsmath}
\usepackage{amsfonts}

\usepackage{breqn}
\newtheorem{definition}{Definition}
\newtheorem{theorem}{Theorem}

\usepackage{subcaption}

\usepackage{color}
\usepackage{bbm}
\usepackage{graphicx}
\graphicspath{{./figures/}} 
\usepackage{mathtools}
\usepackage{float}
\usepackage{subfiles}
\usepackage{algorithm}
\usepackage{caption}
\usepackage{xcolor}
\usepackage{algpseudocode}
\usepackage{float}
\usepackage{multirow}
\usepackage{tabu}
\makeatletter
\newcommand\nocaption{%
    \renewcommand\p@subfigure{}
    \renewcommand\thesubfigure{\thefigure\alph{subfigure}}
}
\newcommand{\thickhline}{%
    \noalign {\ifnum 0=`}\fi \hrule height 1pt
    \futurelet \reserved@a \@xhline
}
\usepackage{booktabs}
\newtheorem{assumption}{Assumption}
\usepackage[utf8]{inputenc}
\usepackage{amssymb}
\newtheorem{proposition}{Proposition}
\title{\LARGE \bf Constraint Inference in Control Tasks from Expert Demonstrations via Inverse Optimization}
\usepackage{times}
\usepackage{hyperref}

\author{Dimitris Papadimitriou and Jingqi Li\thanks{The authors would like to thank Prof. S. Sojoudi for her insightful comments. The authors are with the University of California, Berkeley. Correspondence to \href{mailto:dimitri@berkeley.edu}{\tt dimitri@berkeley.edu}}} 
\date{\today}

\begin{document}

\maketitle
\thispagestyle{empty}
\pagestyle{empty}

\begin{abstract}
    Inferring unknown constraints is a challenging and crucial problem in many robotics applications. When only expert demonstrations are available, it becomes essential to infer the unknown domain constraints to deploy additional agents effectively. 
    In this work, we propose an approach to infer affine constraints in control tasks after observing expert demonstrations. We formulate the constraint inference problem as an inverse optimization problem, and we propose an alternating optimization scheme that infers the unknown constraints by minimizing a KKT residual objective. We demonstrate the effectiveness of our method in a number of simulations, and show that our method can infer less conservative constraints than a recent baseline method, while maintaining comparable safety guarantees.
\end{abstract}


\section{Introduction}
Specifying reward functions for Control and Reinforcement Learning (RL) tasks is a nontrivial process. One approach to specify such reward functions is via inverse Reinforcement Learning (IRL), in which expert demonstrations are used to infer unknown reward functions~\cite{ng2000algorithms}. Once a reward function is obtained, then a policy can be derived. 
Obtaining policies in constrained environments via Control or RL is a fairly manageable task when those constraints are a priori known. In many cases, specifying a reward function is intuitive. For instance, a robot could use the inverse distance to its target destination as a reward function. However, specifying the corresponding constraints for the task at hand is not always straightforward.

One possible avenue to deal with learning policies under unknown constraints is the use of Imitation Learning (IL) algorithms \cite{hussein2017imitation}. Simply imitating an expert agent though, has a number of downsides. First, the inferred policy applies only to the demonstrated task and may not be generalized to new ones. Second, an inferred policy under a specific dynamics model is not guaranteed to be applicable to an agent with different transition dynamics. On the other hand, given the expert demonstrations, we could first infer the constraints and then learn safe policies. The problem of constraint inference in RL is mostly a new and active research area with a number of methods proposed so far \cite{liu2022benchmarking}. 

Constraint inference in control applications is an important problem that has not been extensively studied \cite{chou2020learning}. Given the fact that most control tasks are modeled as an optimization problem, elements from inverse optimization theory can be utilized to infer constraints. In Inverse Optimization (IO), the parameters of an optimization problem are inferred given access to its optimal solution \cite{chan2021inverse}. Although the majority of work revolves around inferring parameters in the objective functions, a parallel to reward learning in RL, research has recently focused on inferring constraint related parameters. 

In this work, we propose a method for inferring constraints in control tasks, based on expert demonstrations. 
Our contributions can be summarized as follows: \textit{1)} We show that the unknown affine constraints can be recovered exactly from the expert trajectory data, when we know the exact number of unknown constraints and {the states for which the constraints are binding}; \textit{2)} When prior knowledge about the time instances at which the constraints are binding is not available, we develop
a method to infer the unknown constraints, by minimizing the KKT residual of the optimal control problem, which could be solved more efficiently than prior works that rely on Mixed Integer Linear Programming (MILP); 
\textit{3)} To deal with the challenge of not knowing the exact number of unknown constraints, we develop a greedy algorithm which sequentially infers an affine constraint such that the KKT residual is minimized. We empirically demonstrate, that our method is robust to observation noise and suboptimal demonstrations and it can infer less conservative constraints than a state-of-the-art baseline method with same-level safety guarantees.

The rest of the paper is organized as follows. In Section \ref{sec:rel_work}, we review the relevant literature. In Section~\ref{sec:prob_form}, we formulate the constraint inference problem. Our main results are presented in Section \ref{sec:cnstr_infer}. In Section \ref{sec:simul}, we demonstrate the performance of our method in simulations. 



\section{Related Work}\label{sec:rel_work}
This section outlines the existing work in constraint inference in RL and Control. We also elaborate on relevant work from the IO literature as this is particularly relevant to our framework.

\textit{Constraint Inverse Reinforcement Learning (CIRL)}: In the context of RL, \cite{scobee2019maximum} proposed a greedy method to incrementally add constraints in discrete state and actions spaces. This is done by assuming access to the reward function and hence, the likelihood of each state or action being constrained can be evaluated. At each iteration, the most likely state or action is classified as a constraint following a maximum likelihood criterion. A continuous state space framework is presented in~\cite{malik2021inverse} in which the authors propose a method to learn a constraint classifier. Finally, in~\cite{papadimitrioubayesian} the authors propose a Bayesian approach to infer constraints. Estimating the full constraint posterior distribution further allows for the quantification of estimation uncertainty.  

\textit{Constraint Inference in Control}:
In the context of robotics, the authors in~\cite{perez2017c} propose a method that infers constraints based on expert demonstrations of a robot interacting with items. The inferred constraints are chosen from a catalog of parametric constraint models so that under them, the robot demonstrations are reproducible.  In~\cite{chou2018learning}, cell occupancies of expert demonstrations are calculated and later used in an integer programming formulation to provide estimates of the constraints. Constraints can also be modeled in the form of control barrier functions. The authors in~\cite{robey2020learning} use expert demonstrations to learn control barrier functions that enjoy provable safety guarantees. 

Moreover, the authors in~\cite{molloy2020online, agrawal2021learning} infer parameters in the objective function utilized to obtain control policies in constraint environments. However, it should be noted that unknown parameters exist only in the objective function and not in the constraints. The KKT optimality conditions of an optimization problem can also be used to infer constraint sets. In~\cite{menner2019constrained}, after candidate constraint sets are constructed, the Lagrange multipliers of the IO problem are used to infer constraints. The closest work to our framework is~\cite{chou2020learning}, in which the authors use the KKT conditions to construct safe and unsafe areas of the state-action space, and hence implicitly the constraints themselves. Although this approach can also infer uncertainty in the objective function, we choose to focus solely on the constraint recovery task.

\textit{Inverse Optimization}: Inverse optimization broadly refers to the problem of inferring parameters of an optimization problem by observing the optimal solution \cite{chan2021inverse, li2023cost}. The majority of the work revolves around estimating parameters in the objective function. For linear objective functions and  constraints, \cite{ahuja2001inverse} formulate the inverse optimization problem also as a linear problem. 
Regarding constraint inference, \cite{chan2020inverse} infer the left-hand side of linear constraints of a linear forward optimization problem. Finally, \cite{ghobadi2021inferring} provides a convex formulation for inferring both the left and right-hand side of linear constraints in linear optimization problems.

In our work, we base our constraint inference method on minimizing the KKT residual of the optimization problem.

\textit{KKT Residual and alternating optimization}:
In optimal control, the KKT conditions are necessary conditions for the optimality of a sequence of states and controls. Recent works \cite{englert2017inverse,awasthi2019forward,menner2020maximum} proposed to minimize the residual of the KKT conditions as a way to find the best cost function such that the computed trajectory matches the expert demonstrations. A benefit of the KKT residual formulation is that it can be convex in the unknown parameters for some classes of problems 
\cite{englert2017inverse,keshavarz2011imputing}. Moreover, the KKT residual formulation can be extended to solve the inverse optimal control problem with inequality constraints on states and controls \cite{englert2017inverse}. 
\section{Problem Formulation}\label{sec:prob_form}
In this work, we are interested in providing a methodology to infer constraints in constrained optimal control problems. We begin by formalizing the classic forward control optimization problem as follows.  
\begin{definition}\label{def:forw_cntrl}
Consider a discrete-time dynamical system ${x}_{t+1}={f}({x_t},{u_t})$  where ${x_t}\in\mathbb{R}^n$ and ${u_t}\in\mathbb{R}^m$ represent the state and control inputs at time $t\in\{0,1,2,\dots\}$, respectively. We denote by $\mathbf{u}=\{u_0,\dots, u_{T-1}\}$ and $\mathbf{x}=\{x_0,x_1,\dots,x_{T}\}$ the control and state trajectories from $t=0$ to $t=T$, respectively. The chosen starting state is denoted with $x_0$. For a control task with a quadratic cost function $c({x},{u}):\mathbb{R}^n\times \mathbb{R}^m\to \mathbb{R}$, terminal cost function $c_T({x}):\mathbb{R}^n\to \mathbb{R}$ and constraint function $g_\theta(x):\mathbb{R}^n\to \mathbb{R}$, parameterized by a vector $\theta\in\mathbb{R}^d$, we define the Forward Control Problem \textbf{FCP}$({\theta})$ as follows,
\end{definition}
\begin{equation}
    \begin{aligned}
        \text{\textbf{FCP}}(\theta):=\min_{\{u_t\}_{t=0}^{T-1}} & \sum_{t=0}^{T-1} c(x_t,u_t)+c_T(x_T)\\
        \textrm{s.t. }& x_{t+1} = f(x_t, u_t), t=0,\ldots,T-1\\
        & g_\theta(x_t)\le 0, t=0,\ldots,T.
    \end{aligned}
\end{equation}
In our work, we focus on state constraints, but our method can be easily utilized for state-action constraints as well. Having access to the optimal solution $\mathbf{u}^*:=\{u_t^*\}_{t=0}^{T-1}$ from the \textbf{FCP} we can formulate the Inverse Control Problem (\textbf{ICP}) as an optimization problem that provides an estimate of the unknown parameters, using a performance metric and constraints that involve the optimal solution. 
\begin{definition}\label{def:inv_cntrl}
Given Definition~\ref{def:forw_cntrl}, we define the Inverse Control Problem as \textbf{ICP}($\mathbf{x}^*,\mathbf{u}^*$)$\coloneqq\min_{{\theta}}$ $\{\ell_{{\mathbf{x}^*,\mathbf{u}^*}}( \theta)$ $|{h}_{\mathbf{x}^*,\mathbf{u}^*}(\theta)\leq {0}\}$ that recovers the unknown constraint parameters ${\theta}$ via an appropriately chosen loss function $\ell$ and constraints $h$, with $\mathbf{x}^*,\mathbf{u}^*$ denoting the optimal state and input solution of \textbf{FCP}. 
\end{definition}
Given these definitions, we are interested in tackling the following problem. Assume we are given a set of trajectories $\{\tau_i\}_{i=1}^{N}$, with each $\tau_i = \{x_0,u_0,x_1,u_1,$ $\dots,x_{T}\}$. What objective function $\ell$ and constraints $h$ should we use in an \textbf{ICP} formulation to recover the unknown constraint parameters ${\theta}$ efficiently? The inferred parameters $\hat{\theta}$ should be such  that the induced optimal trajectories are close to the expert demonstrations and feasible to the unknown constraints. 
\begin{assumption}\label{assum-1}
We assume that we have an expert data set composed by a number of expert state and control input trajectories, which optimally solve the \textbf{FCP}$(\theta)$. 
\end{assumption}
For the remainder of this work, we focus on the case of linear dynamics functions and on finite horizon tasks. More specifically, we consider linear dynamics of the form
\begin{equation}
    x_{t+1} = Ax_t + Bu_t,
\end{equation}
where $x_t\in\mathbb{R}^n$ and $u_t\in \mathbb{R}^m$. Furthermore,  $A\in\mathbb{R}^{n\times n}$ and $B\in\mathbb{R}^{n\times m}$ are considered known matrices. In what follows, we elaborate on the formulation of the \textbf{FCP} and \textbf{ICP} in the case of quadratic cost functions with linear constraints. 
\subsection{Quadratic Objective and Linear Constraints}\label{sec:quad_lin}
We consider control tasks of finite horizon $T$, with linear dynamics, quadratic cost functions and $M$ distinct linear constraints of the following form
    \begin{align*}\label{prob:lqr}  
    \min_{\{u_t\}_{t=0
    }^{t=T-1}}& \sum_{t=0}^{T-1} x_t^\top Q x_t + u_t^\top R u_t +x_T^\top Q x_T\\
    \textrm{s.t. }& x_{t+1} = A x_t + B u_t, \forall t\in\{0,\dots, T-1\} \numberthis \\
    & c_{i}^\top x_t \leq d_i, \forall t \in \{0,\dots,T\}, i\in\{1,\ldots,M\},
    \end{align*}
where $Q$ and $R$ are cost matrices associated with states and control inputs, respectively. Given that we are studying linear constraints, the unknown parameters $\theta$ we are interested in inferring are the $c_i\in\mathbb{R}^n$ and $d_i\in\mathbb{R}$ that parameterize the $i$-th constraint, with $i\in\{1,2,\dots,M\}$. Let $x_0$ be the initial condition and $U=[u_0,u_1\dots,u_{T-1}]^{\top}\in\mathbb{R}^{mT}$ be the compact representation of the control input sequence. The problem can be rewritten in a compact way as
    \begin{align*}\label{eq:LQR_version_1}
    \min_{ U\in\mathbb{R}^{mT}} & (GU+Hx_0)^\top (I_T\otimes Q) (GU+Hx_0) \\ & + U^\top (I_T\otimes R) U\\
    \textrm{s.t.}&\; C(GU+Hx_0) \leq D,\numberthis
    \end{align*}
where $G\in\mathbb{R}^{nT\times mT}$, $H\in\mathbb{R}^{nT\times n}$, $C\in\mathbb{R}^{MT \times nT }$ and $D\in\mathbb{R}^{MT}$ are defined as follows
    \begin{align*}\label{eq:large_mats}
    &G =\begin{bmatrix}
    B & 0 & 0 & \cdots & 0\\
    AB& B & 0 & \cdots & 0\\
    \vdots & \vdots & \vdots & \ddots &\vdots \\
    A^{T-1}B & A^{T-2}B & A^{T-3}B & \cdots &B
    \end{bmatrix},
    H = \begin{bmatrix}
    A\\ A^2 \\ \vdots\\ A^T
    \end{bmatrix},\\
    &C = \begin{bmatrix}I_{T}\otimes c_1^\top\\ \vdots \\I_{T}\otimes c_M^\top\end{bmatrix},\ \ D = \begin{bmatrix}\mathbf{1}_{T} d_1\\ \vdots \\ \mathbf{1}_{T} d_M \end{bmatrix},\numberthis
    \end{align*}
    with $I_n$ denoting the identify matrix and $\mathbf{1}_n$ a vector of ones, both of size $n$. We use for simplicity $0$ in matrix definitions to denote all zero matrices of appropriate dimensions.
\section{Constraint Inference}\label{sec:cnstr_infer}
This section introduces the constraint recovery approach using the inverse problem formulation. Initially, we present conditions for exact constraint recovery. Afterwards, we proceed to deriving a more general approach for constraint inference based on the KKT optimality conditions.
\subsection{Exact Constraint Recovery}
Under certain assumptions, exact constraint inference is possible by solving a linear system of equations. This is possible in the case where the exact number of constraints and the states for which the constraints are binding are known a priori. First, we present conditions under which constraint inference can be exact in the case where the right-hand side of the constraints $d_i,i=1,\ldots,M$ is known. Afterwards, we extend these conditions to the case where the right-hand side is also unknown, a setup that will be kept for the remainder of the paper.
\begin{proposition}\label{prop:exact_rec_known_d}
Assume the existence of $M$ total linear unknown constraints with known right-hand sides $d_i, i=1,\ldots,M$ and that for each  constraint $i=1,\ldots,M$ there are $n_i$, with $T> n_i\geq n$, states for which the constraint is binding at time steps $\mathcal{T}_i$, where $\mathcal{T}_i=\{t|c_i^{\top}x_t=d_i\}$. If, for each constraint, $n$ of these state vectors are linearly independent, then the constraint vector $c_i$ can be recovered exactly, for $i=1,\ldots, M$.
\end{proposition}
\begin{proof}
For each of the constraints $i$ we select $n$ of the $n_i$ states for which the constraint is binding. Then the following holds  
\begin{align}\label{eq:orig_bind_cnstr}
C_{\mathcal{T}_i}\cdot X_{\mathcal{T}_i} = D_{\mathcal{T}_i},
\end{align}
where 
\begin{align}
&C_{\mathcal{T}_i}=\begin{bmatrix}{I}_{n}\otimes c^{\top}_{i} \end{bmatrix}, X_{\mathcal{T}_i}=\begin{bmatrix}x_{t_{i,1}}\\ \vdots \\ x_{t_{i,n}} \end{bmatrix},t_{i,1},\ldots,t_{i,n}\in\mathcal{T}_i\text{ and }\nonumber\\ &D_{\mathcal{T}_i} = \begin{bmatrix} \mathbf{1}_{n}d_i \end{bmatrix}.
\end{align}
Rearranging~\eqref{eq:orig_bind_cnstr} we obtain
\begin{align}
\tilde{X}_{\mathcal{T}_i}\cdot c_i=D_{\mathcal{T}_i},
\end{align}
where 
\begin{align}
\tilde{X}_{\mathcal{T}_i}=\begin{bmatrix} x_{t_{i,1}}^{\top}\\ \vdots \\ x^{\top}_{t_{i,n}}\end{bmatrix}.
\end{align}
Since $\tilde{X}_{\mathcal{T}_i}$ is full row rank given the linear independence of $x_{t_{i,j}}, j=1,\ldots,n$, $c_i$ can be recovered exactly. The non-binding states can be used to determine the sign of the  constraint inequality.
\end{proof}
 For the remainder of the paper we will regard the right-hand side of the constraints unknown as well. To deal with the unknown right-hand side parameter $d_i$, $i\in\{1,\dots,M\}$, we expand the state dimension by one element that corresponds to the right-hand side of the constraints to  simplify notation. More specifically, $x_t:=[x_t,1]^{\top}\in\mathbb{R}^{n+1}$ and
 \begin{align}
 A:=\begin{bmatrix}A&0\\0&1\end{bmatrix}, B:=\begin{bmatrix}B\\0\end{bmatrix}.
 \end{align}
 Now each constraint can be compactly written as $c_i^{\top}x_t\leq 0$ where $c_i^{\top}:=[c_i^{\top},-d_i]\in\mathbb{R}^{n+1}$. Exact constraint inference in this case can occur in a similar way to Proposition~\ref{prop:exact_rec_known_d} and under the same assumptions. More specifically, the right-hand-side of the constraints can be set equal to an arbitrary value and the inferred constraints will be inferred up to scale. 
The above result suggests that, we can infer the exact constraints under the strong assumptions that we know the total number of constraints and for which states the constraints are binding. In this work, we are primarily interested in providing a framework for constraint inference in challenging scenaria in which the actual number of constraints and the time steps at which the constraints are binding, if at all, are unknown. To solve this challenging problem, we utilize the following KKT residual objective formulation, which is a general framework that allows for constraint inference in the most general settings.
\subsection{General Constraint Inference via KKT Residual}\label{sec:KKT_res}
In this section we focus on the case where the actual number of constraints $M$ is known a priori. We will be extending our approach to the case where $M$ is unknown in section~\ref{sec:greedy_alg}. {In what follows, we formulate the Lagrangian for one demonstration from one initial condition. This setup can be easily extended to multiple demonstrations from distinct initial conditions by taking the sum of their Lagrangians.
} The optimization problem in~\eqref{eq:LQR_version_1} admits the following Lagrangian
\begin{align}
    &L(U,\lambda):=(GU+Hx_0)^\top (I_T\otimes Q)(GU+Hx_0)\nonumber\\&+U^\top (I_T\otimes R)U
    +\lambda^\top C(GU+Hx_0),
\end{align}
with $\lambda\in\mathbb{R}_+^{MT}$ denoting the vector of Lagrange multipliers. The KKT conditions of stationarity, complementary slackness, primal and dual feasibility are then
\begin{align}
    &\nonumber \nabla_U L(U^*,\lambda) = 2G^\top (I_T\otimes Q)(GU^*+Hx_0)+\\
    &2 (I_T\otimes R)U^*+G^\top C^\top\lambda=0 \\ 
    &\lambda^\top C(GU^*+Hx_0)=0\\
    &C(GU^*+Hx_0)\le 0\\
    &\lambda\ge 0.
\end{align}
The KKT residual formulation obtains estimates of the unknown parameters, in our case $c_i$, $i=1,\ldots,M$ and $\lambda$, by minimizing some metric of the stationarity and complementary slackness conditions, while satisfying primal and dual feasibility. More specifically,
\begin{equation}\label{opt:kkt_res}
    \begin{aligned}
    \min_{\{c_i\}_{i=1
    }^{i=M}, \lambda}&\; \ell(\{c_i\}_{i=1
    }^{i=M}, \lambda)\\
    \textrm{s.t. }& \; C_{\{c_i\}_{i=1}^M}A_2\leq 0\\
    & \lambda\geq 0,
    \end{aligned}
\end{equation}
where for conciseness we define
\begin{align}
&\ell(\{c_i\}_{i=1
    }^{i=M}, \lambda)=\|A_1+G^{\top}C_{\{c_i\}_{i=1}^M}^{\top}\lambda\|_2^2\nonumber \\&+ \| (C_{\{c_i\}_{i=1}^M}A_2)^{\top}\lambda \|_2^2 + \rho_1 \sum_{i=1}^{M}\|c_i\|_2^2+ \rho_2\|\lambda\|_2^2\label{eq:objective_KKT}\\
  &A_1= 2G^\top (I_T\otimes Q)(GU^*+Hx_0)+2 (I_T\otimes R)U^*\\
   &A_2= GU^*+Hx_0,
\end{align}
with $\rho_1,\rho_2$ being regularization parameters. Given that the constraints $c_i$ are now unknown, we use the notation $C_{\{c_i\}_{i=1}^M}$ to denote the function that maps the $c_i$s as follows
\begin{align}
C_{\{c_i\}_{i=1}^M} := \begin{bmatrix}I_{T}\otimes c_1^\top\\ \vdots \\I_{T}\otimes c_M^\top\end{bmatrix}.
\end{align}
It should be noted that this is a biconvex optimization problem in $c_i, {i=1,\ldots,M}$ and $\lambda$~\cite{gorski2007biconvex}. The most common approach for solving biconvex optimization problems, is via an alternating optimization scheme, which we outline in the following section. 
\subsection{Alternating Optimization Formulation}\label{sec:alter_opt}
Observing Problem~\eqref{opt:kkt_res}, it is evident that by fixing one of the variables the problem becomes a constrained least squares problem with respect to the other variable. Without loss of generality, in this section, we focus on the case where $M=1$ with the only constraint being $c\in\mathbb{R}^{n+1}$. More specifically, by fixing {$c$} Problem~\eqref{opt:kkt_res} becomes 
\begin{equation}\label{eq:min_lambda}
    \begin{aligned}
    &\min_{\lambda}\; \|\tilde{A}\lambda-\tilde{b}\|_2^2 + \rho_2\|\lambda\|_2^2 \\
    & \textrm{s.t.}\;\lambda \geq {{0}},
    \end{aligned}
\end{equation}
where
\begin{align*}
    \tilde{A}=\begin{bmatrix} G^{\top}C_{\{c\}}^{\top}\\(C_{\{c\}}A_2)^{\top}
    \end{bmatrix},\tilde{b}=\begin{bmatrix}-A_1\\0\end{bmatrix} \text{ and } C_{\{c\}}=[I_T\otimes c].
\end{align*}
On the other hand, when minimizing over $c$, problem~\eqref{opt:kkt_res} can be rewritten as
\begin{equation}\label{eq:min_c}
    \begin{aligned}
    &\min_{c}\; \|\tilde{\tilde{A}}c-\tilde{\tilde{b}}\|_2^2 +\rho_1 ||c||_2^2\\
     &\textrm{s.t. }\; C_{\{c\}}A_2\leq {0},
    \end{aligned}
\end{equation}
where
\begin{align*}
    \tilde{\tilde{A}}=\begin{bmatrix} G^{\top}\Lambda\\A_2^{\top}\Lambda
    \end{bmatrix},\tilde{\tilde{b}}=\begin{bmatrix}-A_1\\0\end{bmatrix},
\end{align*}
and ${\Lambda}$ is a constructed sparse matrix with elements the coordinates of $\lambda$, such that $C_{\{c\}}^{\top}\lambda={\Lambda}c$. The reformulation using matrix  $\Lambda$ is solely done to show that Problem~\ref{eq:min_c} is indeed a least squares problem on $c$, with constraints that are also linear in $c$. 

Problems~\eqref{eq:min_lambda} and~\eqref{eq:min_c} included in the objective function a regularization term to ensure a unique solution. The detailed process of the alternating optimization approach can be seen in Algorithm~\ref{algo:alter}. For a number of $K$ iterations, we solve Problems~\eqref{eq:min_lambda} and~\eqref{eq:min_c} in an alternating way. 
\begin{algorithm}[h]
\caption{Alternating Minimization Algorithm}\label{alg:alg_alter}
\begin{algorithmic}[1]
\State \textbf{Initialize randomly:} $c\in\mathbb{R}^{n+1}$
\For{$K =1,\ldots$}
    \State Solve Problem~\eqref{eq:min_lambda}
    \State Update $\lambda$ value
    \State Solve Problem~\eqref{eq:min_c}
    \State Update $c$ value
\EndFor
\State \textbf{Output:} $c, \lambda$
\end{algorithmic}
\label{algo:alter}
\end{algorithm}
\begin{theorem}\label{thm:conv}
Algorithm~\ref{alg:alg_alter} converges to a stationary point of \eqref{opt:kkt_res}.
\end{theorem}
\begin{proof}
For the theorem to hold, Problem~\eqref{eq:min_lambda} and  Problem~\eqref{eq:min_c} must have unique solutions~\cite{bertsekas1999nonlinear}. Problem~\eqref{eq:min_lambda} is a regularized constrained least squares problem in which the objective function is strictly convex.
Thus, Problem~\eqref{eq:min_lambda} admits an unique solution. Similarly, Problem~\eqref{eq:min_c}  also admits a unique solution due to strong convexity. Convergence then follows from Proposition 2.7.1 in~\cite{bertsekas1999nonlinear}. The result is shown for the case when we infer only one constraint because, as it will become clear is subsequent sections, we utilize Algorithm~\ref{algo:alter} to infer one constraint at a time. 
\end{proof}
\subsection{Greedy Constraint Inference}\label{sec:greedy_alg}
The previous sections, have tackled the scenario of unknown constraints with an a priori known number of constraints in the environment. However, in realistic scenaria, such cases are rare and thus we need a heuristic approach that infers constraints when their number is unknown. In that direction, we propose using the KKT residual objective as an estimate of the quality of constraint inference. In what follows, we propose a greedy constraint inference algorithm, that incrementally adds constraints until a certain criterion is met.

The Incremental Greedy Constraint Inference (IGCI) approach, shown in
Algorithm~\ref{alg:al_greedy}, utilizes at each iteration Algorithm~\ref{alg:alg_alter} to incrementally add a constraint in the set of inferred constraints. The method starts with an empty set of constraints $\mathcal{C}$. After a constraint $c_i$ and the Lagrange multiplier have been inferred following Algorithm~\ref{alg:alg_alter}, the constraint is temporarily added to the constraint set $\mathcal{C}$. Then the KKT residual objective $\ell(\{c_i\}_{c_i\in\mathcal{C}
    }, \lambda)$ is evaluated. If the difference between this value and the KKT residual in the previous iteration, which is denoted with $E$ in Algorithm~\ref{alg:al_greedy}, is above a user-specified threshold $\delta$, then the constraint $c_i$ remains in the constraint set. The process is repeated until this difference drops below the predefined threshold at which point the process terminates, without including the last inferred constraint in $\mathcal{C}$.
\begin{algorithm}[h]
\caption{IGCI}
\begin{algorithmic}[1]
\State \textbf{Parameters:} Convergence threshold $\delta$, $E$\(>\!\!>\)$1$
\State \textbf{Initialize:} $\mathcal{C}'=\{\}$
\For{$N_c =1,\ldots$}
\State Obtain $c_{N_c}\in\mathbb{R}^{n+1}$ and $\lambda\in\mathbb{R}^{N_cT}$ from Algorithm~\ref{alg:alg_alter}
\State $\mathcal{C}'\gets\mathcal{C}'\cup \{c_{N_c}\}$
\State $\lambda'\gets\lambda$
\If{$|\ell(\{c\}_{c\in\mathcal{C}'}, \lambda')-E|\leq \delta$ }
      \State break
\Else
    \State $\mathcal{C}\gets\mathcal{C}', \;\lambda\gets\lambda'$
    \State $E\gets \ell(\{c\}_{c\in\mathcal{C}}, \lambda)$
\EndIf
\EndFor
\State \textbf{Output:} $\mathcal{C}, \lambda$
\end{algorithmic}
\label{alg:al_greedy}
\end{algorithm}
When the algorithm terminates, the set of $N_c$ inferred constraints $\mathcal{C}=\{\tilde{c}_1,\ldots,\tilde{c}_{N_c}\}$ is returned. It should be noted that $N_c$ need not be equal to the original number of constraints $M$. Furthermore, IGCI will add at least one constraint given that the criterion in line $7$ will not be satisfied in the first iteration. This can be modified if needed, by appropriately choosing a smaller initial value for $E$. Finally, in line $4$, {Algorithm~\ref{algo:alter} can be run multiple independent times in order to obtain a better solution.}

At an arbitrary iteration $k$ of IGCI, in line $4$ we obtain the estimate for an additional constraint and Lagrange multiplier. This is achieved by solving~\eqref{opt:kkt_res} using Algorithm~\ref{alg:alg_alter} with the unknowns being $c_k$ and the multiplier $\lambda\in\mathbb{R}^{kT}$, while fixing the values of the already inferred constraints, $\tilde{c}_i,\;i=1,\ldots,k-1$. Although further fixing the already inferred Lagrange multipliers and only inferring an additional $\lambda\in\mathbb{R}^T$ at each iteration works well, our simulations showed that inferring at iteration $k$ the Lagrange multipliers for all constraints $i=1,\ldots,k$ gives slightly more accurate results.

Finally, the analysis so far has been based on a single optimal demonstration $U^*$. In subsequent sections, we allow for inference from multiple demonstrations by modifying $\tilde{A}, \tilde{b}, \tilde{\tilde{A}}$ and $\tilde{\tilde{b}}$ in problems~\eqref{eq:min_lambda} and~\eqref{eq:min_c} appropriately. More specifically, for $N$ demonstrations $U^*_1,\ldots,U^*_N$ {starting from states $x_0^{(1)},\ldots,x_0^{(N)}$ }and by denoting
\begin{align}
  &A_1^{(i)}= 2G^\top (I_T\otimes Q)(GU_i^*+H{x_0^{(i)}})+2 (I_T\otimes R)U_i^*\\
   &A_2^{(i)}= GU_i^*+H{x_0^{(i)}},
\end{align}
the new $\tilde{A}$ and $\tilde{b}$ for Problem~\ref{eq:min_lambda} now become
\begin{align*}
    \tilde{A}=\begin{bmatrix} G^{\top}C_{\{c\}}^{\top}\\(C_{\{c\}}A^{(1)}_2)^{\top}\\
    \vdots\\
    G^{\top}C_{\{c\}}^{\top}\\(C_{\{c\}}A^{(N)}_2)^{\top}
    \end{bmatrix},\tilde{b}=\begin{bmatrix}-A_1^{(1)}\\0\\ \vdots\\-A_1^{(N)}\\0\end{bmatrix}.
\end{align*}
The definitions for $\tilde{\tilde{A}}$ and $\tilde{\tilde{b}}$ are analogous.

{In this paper, we utilize demonstrations whose starting states $x_0^{(i)}$ are sampled by injecting normal noise to a chosen $x_0$ state. The latter is chosen so that sufficient interaction with the constraints is attained. Evidently, if there is no interaction between the demonstrations and the constraints, inference can be inaccurate. Furthermore, we can approximate the Lagrange multipliers associated with each trajectory by a single set of Lagrange multipliers as the demonstrations in our work are close to each other, leading to improved computational efficiency. In case initial conditions spanning the entire state space are used, which can effectively lead to better exploration of the state space, individual Lagrange multipliers for each demonstration would be required.}
\section{Simulations}\label{sec:simul}
In this section, we carry out simulations to quantify the performance of IGCI. 
We utilize a two dimensional navigation task and a robotic arm environment and consider scenaria of perfect and noisy state observations, as well as suboptimal trajectories. To evaluate the quality of the inferred constraints, we utilize the following metrics. Similar to~\cite{chou2020learning}, we report the coverage of the actual safe region and the overlap with the unsafe one. More specifically, we define the constrained and unconstrained regions as
\begin{align}\label{eq:cnstr-uncnstr-regions}
\mathcal{A}&= \bigcup_{ i=1,\ldots,M}\{x| c_{i}^\top x> 0\}\\\label{eq:cnstr-uncnstr-regions2}
\mathcal{A}^c&=\bigcap_{ i=1,\ldots,M}\{x| c_{i}^\top x\leq 0\},
\end{align}
respectively. The corresponding regions constructed using the estimated parameters $\tilde{c}_i, i=1,\ldots,N_c$ are designated with $\tilde{\mathcal{A}}$ and $\tilde{\mathcal{A}}^c$, respectively.  We now define the coverage of the unconstrained set as $\frac{Vol(\tilde{\mathcal{A}^c}\cap \mathcal{A}^c)}{Vol(\mathcal{A}^c)}$  and the overlap of the inferred unconstrained and the actual constrained sets as $\frac{Vol(\tilde{\mathcal{A}^c}\setminus\mathcal{A}^c)}{Vol(\mathcal{A})}$. These metrics can be seen as analogous to the True Positive (TP) and False Positive (FP) classification rates, {and hence they will be called as such}. 
\subsection{2D Navigation Task}
In this set of simulations we study a two dimensional linear system with $x_t\in\mathbb{R}^2, u_t\in\mathbb{R}$, dynamics matrices 
\begin{align}
A=\begin{bmatrix}
0.8&0.1\\0.1&0.8
\end{bmatrix}, B=\begin{bmatrix}
0.1\\0.5
\end{bmatrix},
\end{align}
and quadratic  cost matrices  $Q=100\cdot I_2$ and $R=0.1$.
The goal of the controller is to stay close to a predetermined trajectory $x_{track}$, listed in Appendix~\ref{sec:sim_detail}, for the duration of the task. We evaluate Algorithm~\ref{algo:alter} in an environment with two linear state constraints, shown in Figures~\ref{fig:2_c_no_noise} and \ref{fig:2_c_noise}. The control task has a horizon $T=8$ and the starting states are normally distributed around $[-0.5,-5.5]^{\top}$ with $\sigma=0.05$. We consider two cases, one with perfect and one with imperfect state observations. In all simulations, we vary the number of expert demonstrations $N$, in order to quantify its impact on inference. IGCI is terminated after $2$ iterations as for subsequent iterations the KKT cost function, as seen in Figure~\ref{fig:conv_rate_2d}, no longer decreases significantly. For all simulations in this section, we followed the elbow rule for the termination of the algorithm, without explicitly choosing a threshold $\delta$. For each experiment, we carry out $10$ independent simulations with newly obtained expert demonstrations. 

We first investigate constraint inference under perfect state observations. In Figure~\ref{fig:2_c_no_noise}, we plot the expert demonstrations along with the actual and estimated constraint regions, for $N=1$ and $N=10$. We denote with $OC$ and $OCR$ the original affine constraint and corresponding region, respectively. $EC$ and $ECR$ designate the estimated affine constraint and constraint regions. 
\begin{figure}[h]
\begin{subfigure}[t]{.24\textwidth}\centering
\includegraphics[width=1.0\columnwidth]{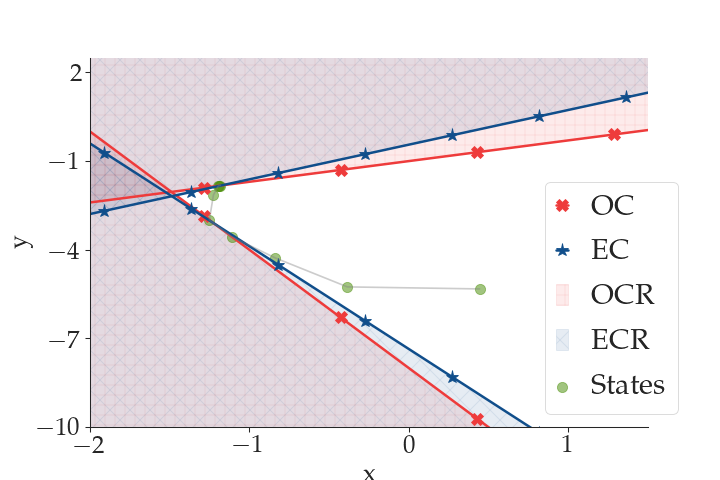}
\caption{$N=1.$}
\label{fig:2d_noiseless_1}
\end{subfigure}%
\hspace{-0.4cm}
\begin{subfigure}[t]{.24\textwidth}\centering
\includegraphics[width=1.0\columnwidth]{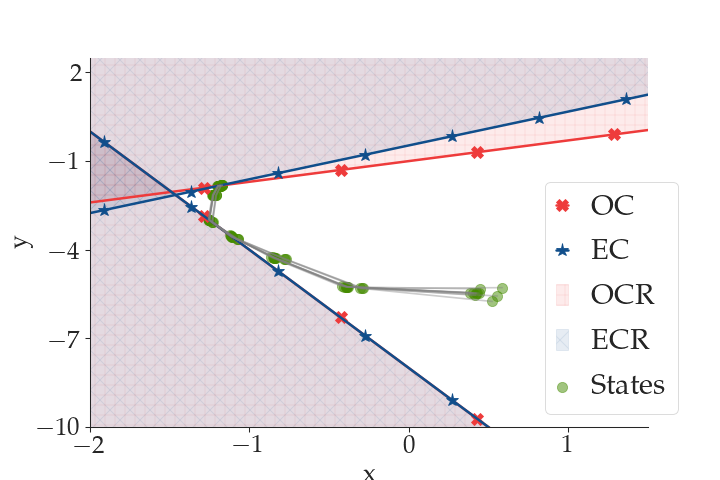}
\caption{$N=10.$}
\label{fig:2d_noiseless_5}
\end{subfigure}%
\caption{Examples of original and estimated constrained regions for $N=1$ (Fig.~\ref{fig:2d_noiseless_1}) and $N=10$ (Fig.~\ref{fig:2d_noiseless_5}) along with expert demonstrations under perfect state observations.}
\label{fig:2_c_no_noise}
\end{figure}

In the case of noisy state observations, we assume that each demonstration $i$ is polluted with observation noise as follows
\begin{align}
    \tau_i = \{x_0,u_0,x_1+\epsilon,u_1,x_2+\epsilon,\dots,u_{T-1},x_{T}+\epsilon\},
\end{align}
with $\epsilon\sim\mathcal{N}(0,0.005)$. In Figure~\ref{fig:2_c_noise} we plot the original constraints along with an instance of the inferred ones for $N=1$ and $N=10$, respectively. For $N=1$, we showcase one of the $10$ simulations in which inference for one of the constraints was inaccurate. {Clearly, for both perfect and noisy state observations, a larger number of demonstrations, which possibly results in more interactions with the constraints, leads to more accurate inference.} Figure~\ref{fig:class_vols_2d} contains the TP and FP rates, as discussed in the next section.
\begin{figure}[h]
\begin{subfigure}[t]{.24\textwidth}\centering
\includegraphics[width=1.0\columnwidth]{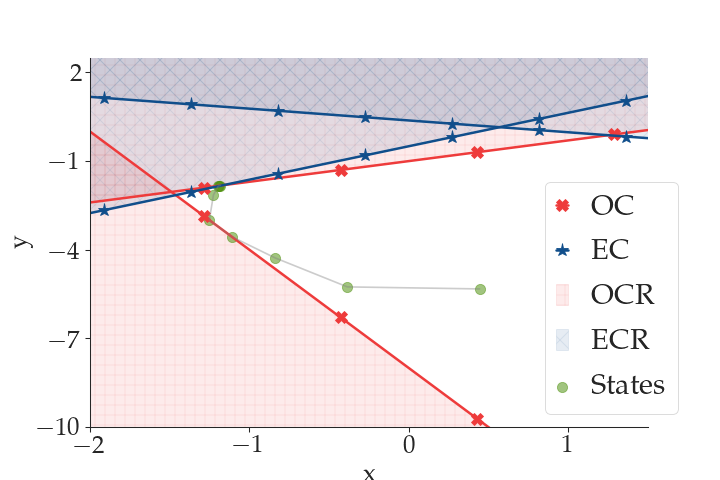}
\caption{$N=1.$}
\label{fig:2d_n1_noise}
\end{subfigure}%
\hspace{-0.4cm}
\begin{subfigure}[t]{.24\textwidth}\centering
\includegraphics[width=1.0\columnwidth]{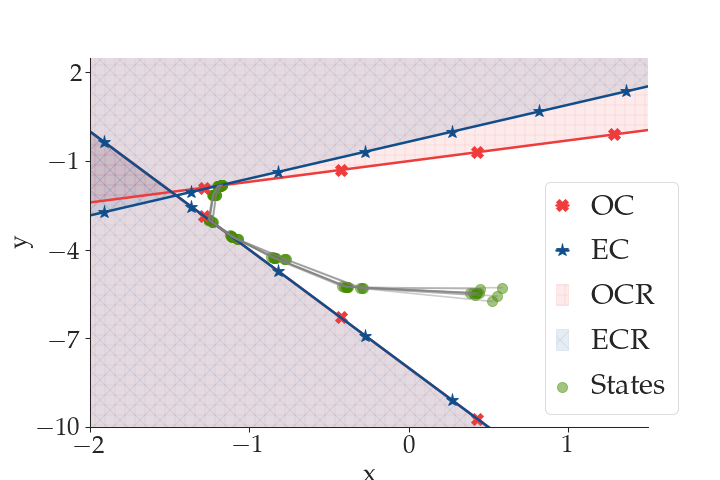}
\caption{$N=10.$}
\label{fig:2d_n5_noise}
\end{subfigure}%
\caption{Examples of original and estimated constrained regions for $N=1$ (Fig.~\ref{fig:2d_n1_noise}) and $N=10$ (Fig.~\ref{fig:2d_n5_noise}) along with expert demonstrations under noisy state observations.}
\label{fig:2_c_noise}
\end{figure}
\vspace{-0.5cm}
\subsection{Baseline Comparison}
To emphasize the benefits of inferring constraint parameters and not just safe/unsafe regions, we compare IGCI with the inference method presented in~\cite{chou2020learning}. The latter, infers regions of the state space that are deemed constrained or unconstrained. We solve Problem 5 from~\cite{chou2020learning} in order to obtain an estimate of the unconstrained region for our two dimensional navigation task, in the cases of perfect and imperfect state observations. We utilize $10$ different query points to obtain the estimate of the unconstrained set. 
\begin{figure}[h]
\centering
\includegraphics[width=0.33\textwidth]{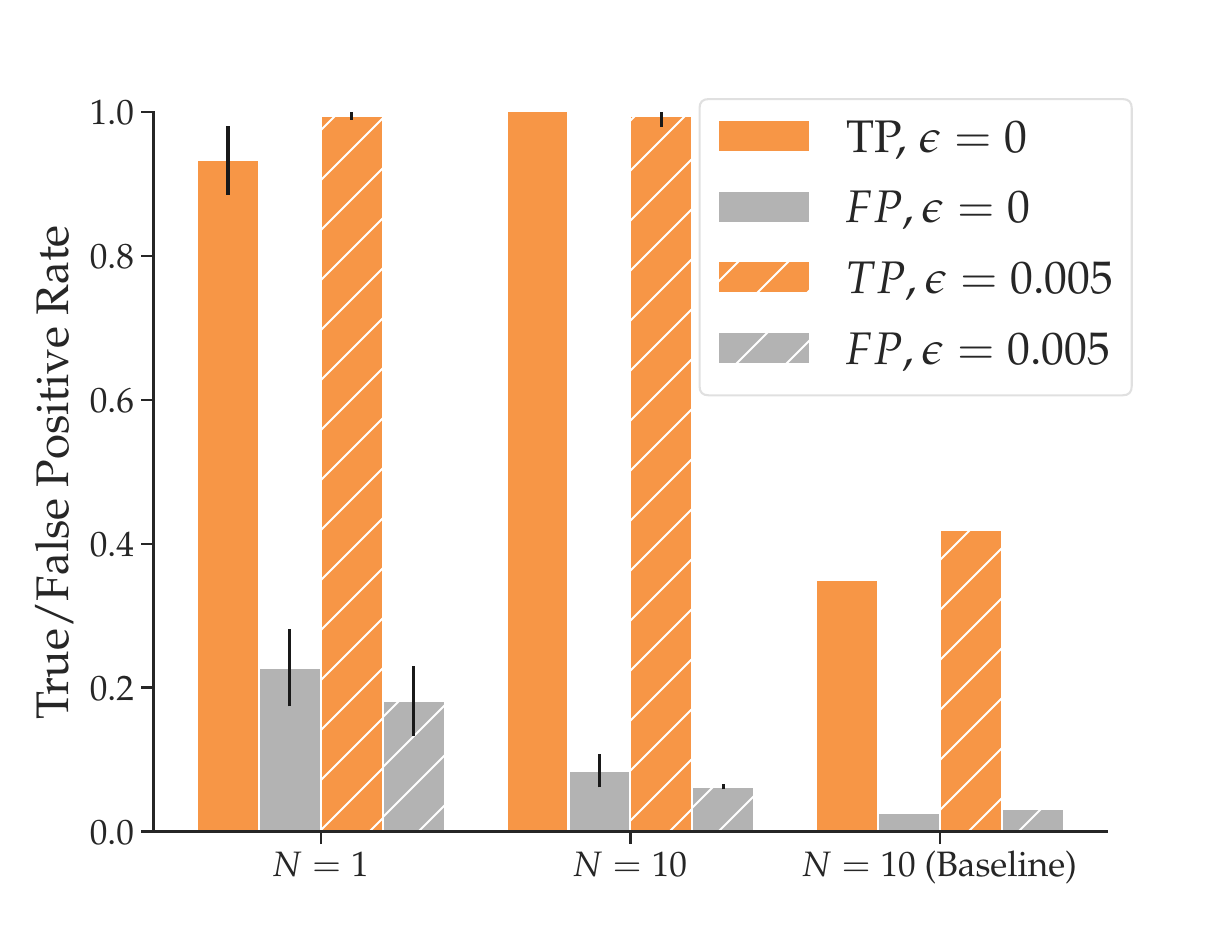}
\label{fig:viol_noiseless}
\caption{Classification rates for IGCI and baseline under perfect and noisy state observations. Results averaged over $10$ simulations. When $N=10$, our method has significantly higher true positive rate than the baseline method while maintaining same-level of false positive rates.}
\label{fig:class_vols_2d}
\end{figure}
Evaluation of the algorithms is carried out using the TP and FP classification rates. 

The results for both the noiseless and noisy state observations
are given in Figure~\ref{fig:class_vols_2d}. Increasing the number of available trajectories naturally leads to better estimation, which by itself allows for safer policies. The baseline method, which we only implemented in the more favorable case of $N=10$ demonstrations, only manages to recognize rectangular areas around the trajectories and hence fails to identify the majority of the unconstrained regions ``to the right" of the demonstrations. On the other hand, our approach manages to infer the majority of the unconstrained region while having a false positive rate similar to the more conservative baseline. It should be stressed, that our method returns the full parameterization of the constraints and hence a new forward optimization problem is still convex. On the other hand, the baseline method infers rectangular regions, the union of which need not be a convex set, hence, requiring more complex techniques for novel forward problems in the inferred domain. Finally, the baseline method requires the solution either of a nonlinear program or a MILP, which can be significantly more inefficient when compared to our alternating approach.
\subsection{3D Manipulation Task with Suboptimal Demonstrations}
In this section, we utilize the Fetch-Reach robotic simulation~\cite{gymnasium_robotics2023github} which is designed to complete tasks, like pushing and grabbing objects, in three dimensional environments. The control inputs are the desired displacement of the robot gripper in the $x,y$ and $z$ coordinates, $x_t\in\mathbb{R}^3, u_t\in\mathbb{R}^3$. We assume that the nominal dynamics matrices, the ones {the end effector} of the robot theoretically adheres to,  are $A={I}_3$ and $B={I}_3$ and the quadratic cost matrices are $Q=10\cdot{I}_3$ and $R=10\cdot{I}_3$. 

It should be noted that the nominal dynamics are a simplification of the actual dynamics, as the underlying physics engine is far more complicated and for a particular control input the arm gripper will not land in the predicted state by our nominal model but at a state in the proximity of the former. For this reason, we study two cases, one with nominal trajectories and one with the actual robot trajectories. To obtain the latter, after obtaining an optimal trajectory by solving~\eqref{prob:lqr} using the nominal dynamics, we utilize a closed loop controller between optimal waypoints to guide the gripper close to the nominal states (robot states). The controller is run for $5$ iterations for each waypoint and its action is determined by the difference between its current state and the waypoint. Clearly, these trajectories, shown in Figure~\ref{fig:3D_traj}, are no longer optimal according to~\eqref{prob:lqr}. 

The control task is for the end effector to stay close to a tracking trajectory, which can be found in Appendix~\ref{sec:sim_detail}, over a horizon of $T=8$. The ground truth underlying constraints, schematically pictured in Figure~\ref{fig:fetch}, are $x\geq 0.9$ and $y\leq 1.1$, which can also be written as $c_i^{\top}x\leq0,\;i=1,2$ with $c_1=[-1,0,0,0.9]^{\top}$ and $c_2=[0,1,0,-1.1]^{\top}$. The boundaries of the entire domain are $x\in[0.4,1.4]^{\top}, y\in[0.4,1.4]^{\top}$ and $z\in[0.4,1.4]^{\top}$. We gather $N=1$ and $N=10$ expert trajectories with the starting states being normally distributed around $[1,1,0.8]$ with $\sigma=0.05$.
\begin{figure}[t!]
\begin{subfigure}[t]{.25\textwidth}\centering
  \includegraphics[trim={0cm 0cm 0cm 0.5cm},clip,width=.7\columnwidth]{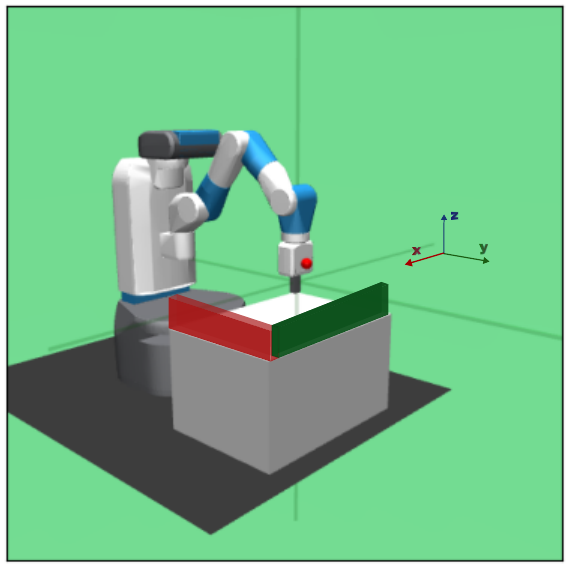}
  \caption{Fetch-Reach robot.}
  \label{fig:fetch}
\end{subfigure}%
\begin{subfigure}[t]{.25\textwidth}\centering
\hspace*{-1.08cm} 
  \includegraphics[trim={3.05cm 1.0cm 0cm 0cm},clip,width=1.006\columnwidth]{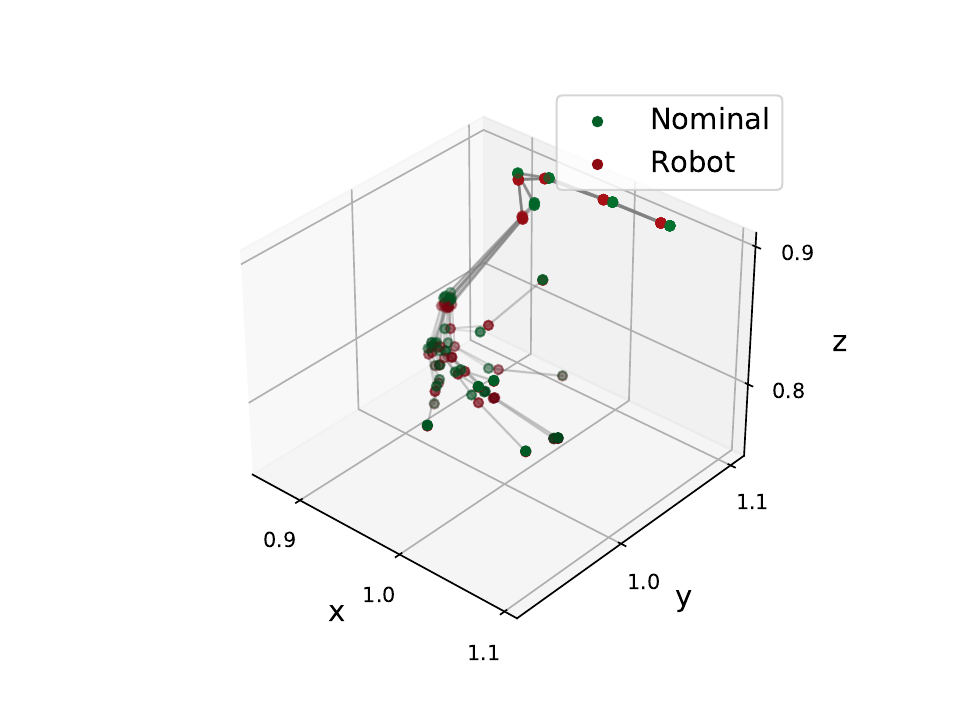}
  \caption{End effector trajectories.}
   \label{fig:3D_traj}
\end{subfigure}
\caption{(\ref{fig:fetch}) Fetch-Reach robot along with constraints (red, green). (\ref{fig:3D_traj}) Nominal and corresponding actual robot trajectories for $N=10$.}
\label{fig:3D_example}
\end{figure}
Figure~\ref{fig:3D_traj} shows $N=10$ trajectories in which the green points correspond to the nominal trajectories and the red ones are the actual robot states. We infer the constraints by using both  the nominal and the suboptimal demonstrations, in order to quantify the performance of our method in view of suboptimality.  We run IGCI for $10$ independent simulations and we report the average classification rates in Figure~\ref{fig:vols_3d}. Clearly, estimation is effective even by utilizing a single expert demonstration. 
\begin{figure}[t!]
\centering
\includegraphics[trim = 0cm 1cm 0cm 1cm, width=0.31\textwidth]{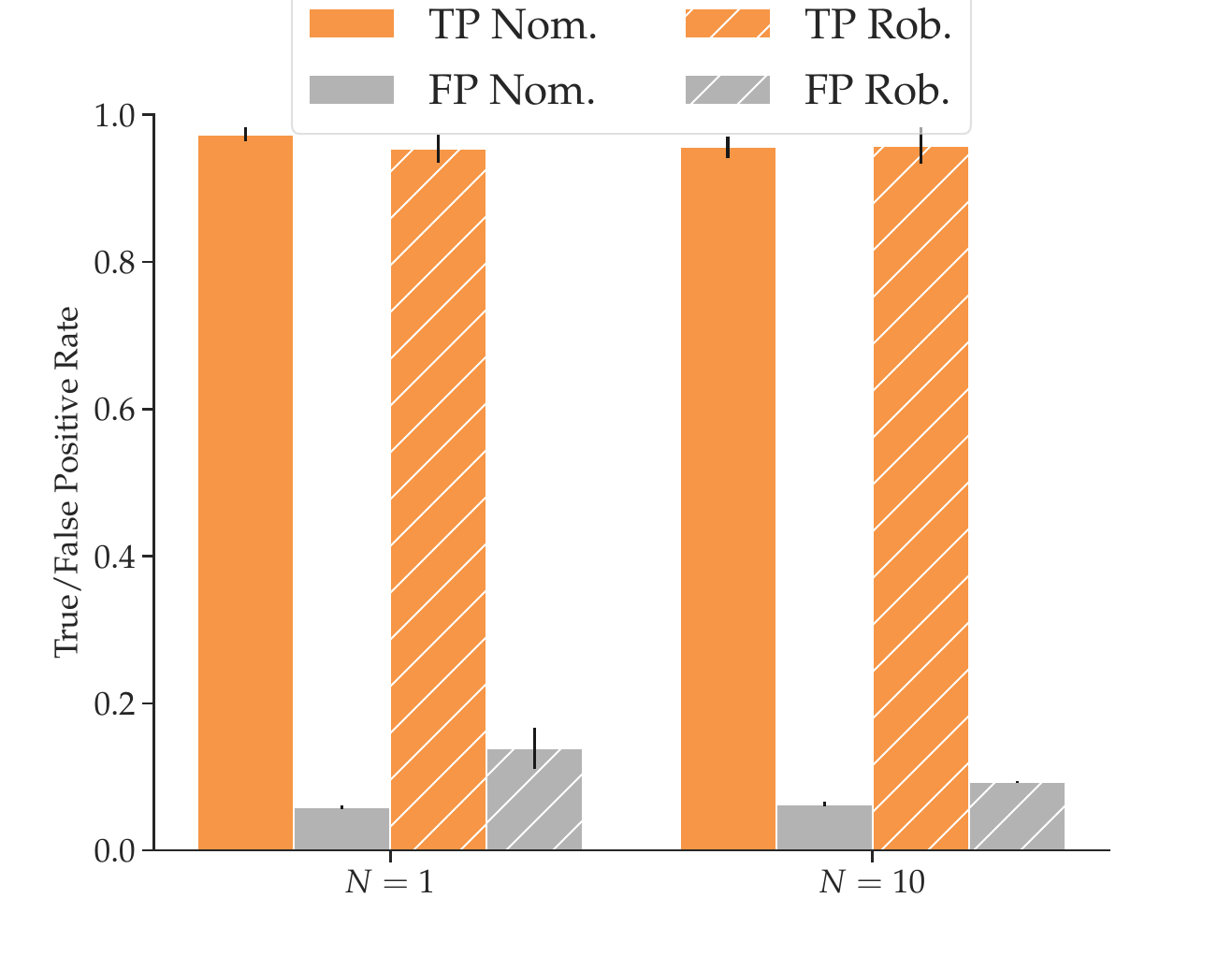}
\label{fig:viol_noiseless}
\caption{Classification rates for nominal and robot trajectories. The robot trajectories are considered to be suboptimal, because of the model discrepancy between the nominal dynamics and the real robot dynamics in the Fetch-Reach simulation environment \cite{gymnasium_robotics2023github}. Results are averaged over $10$ independent simulations. Our method is still effective even when the trajectories are suboptimal.}
\label{fig:vols_3d}
\end{figure}
In Figure~\ref{fig:3d_cnstrs} we plot the actual and the inferred constraints as obtained for $N=10$ nominal trajectories. Figure~\ref{fig:conv_rate_3d} displays the convergence rate of the KKT residual objective with respect to the number of iterations. As expected, for optimal demonstrations the KKT cost function is lower than the corresponding for suboptimal  demonstrations.
\begin{figure}[H]
\centering
\includegraphics[trim={0cm 0.0cm 0 0cm},clip,width=0.27\textwidth]{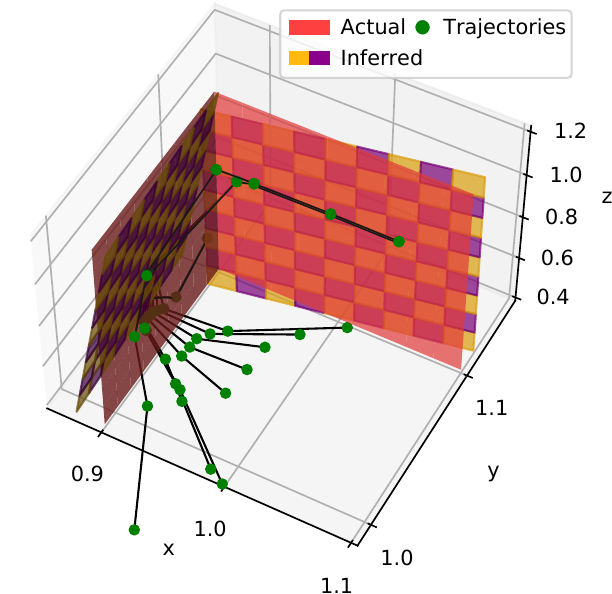}
\label{fig:viol_noiseless}
\caption{Nominal trajectories ($N=10$) of the end effector along with actual and inferred constraints. The inferred constraints align well with the actual constraints.}
\label{fig:3d_cnstrs}
\end{figure}
\begin{figure}[H]
\begin{subfigure}[t]{.5\columnwidth}
\hspace{0.3cm}\includegraphics[width=.76\columnwidth]{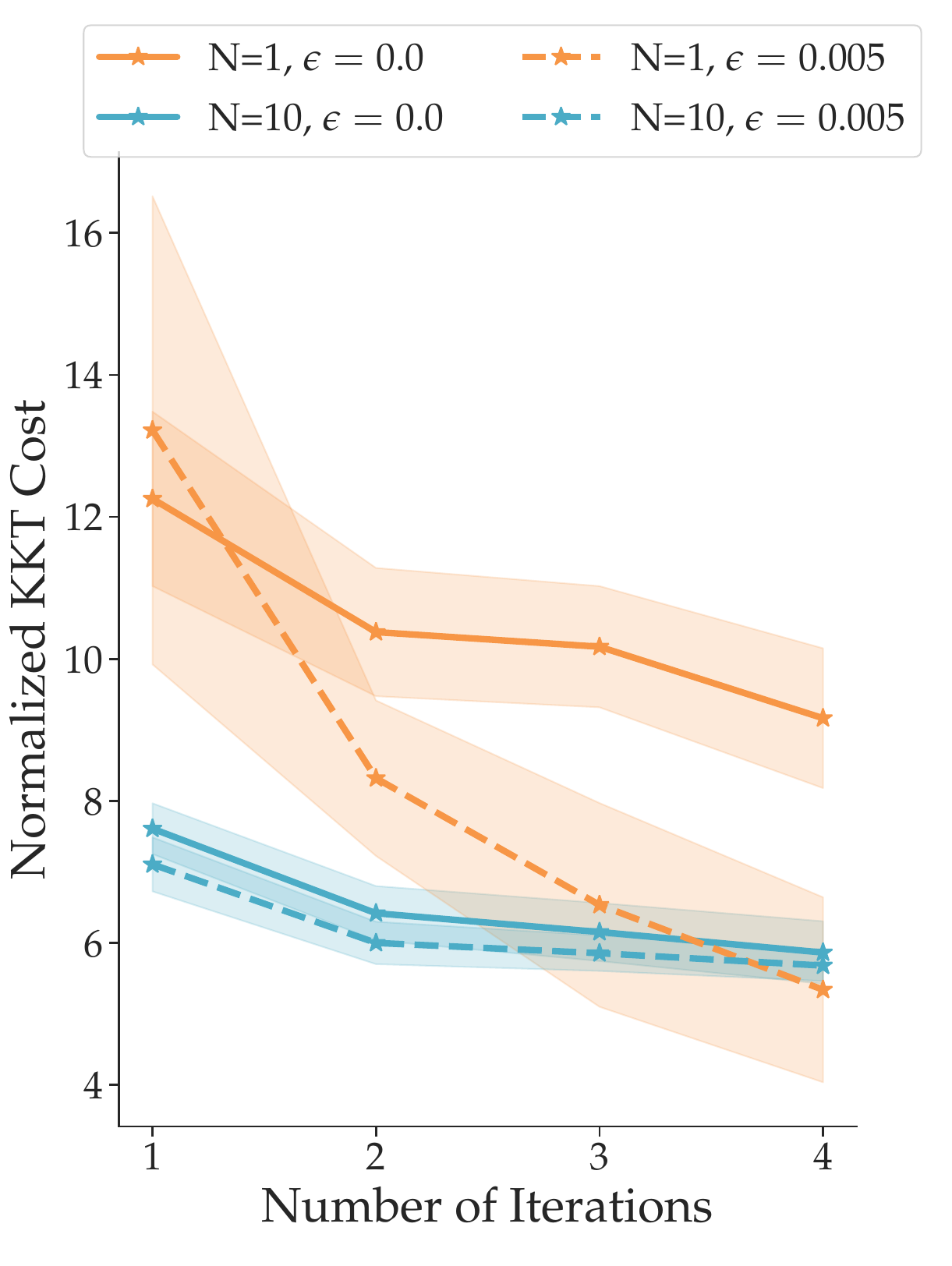}
\caption{}
\label{fig:conv_rate_2d}
\end{subfigure}
\hspace{-0.5cm}
\begin{subfigure}[t]{.5\columnwidth}
\hspace{0.3cm}\includegraphics[width=.74\columnwidth]{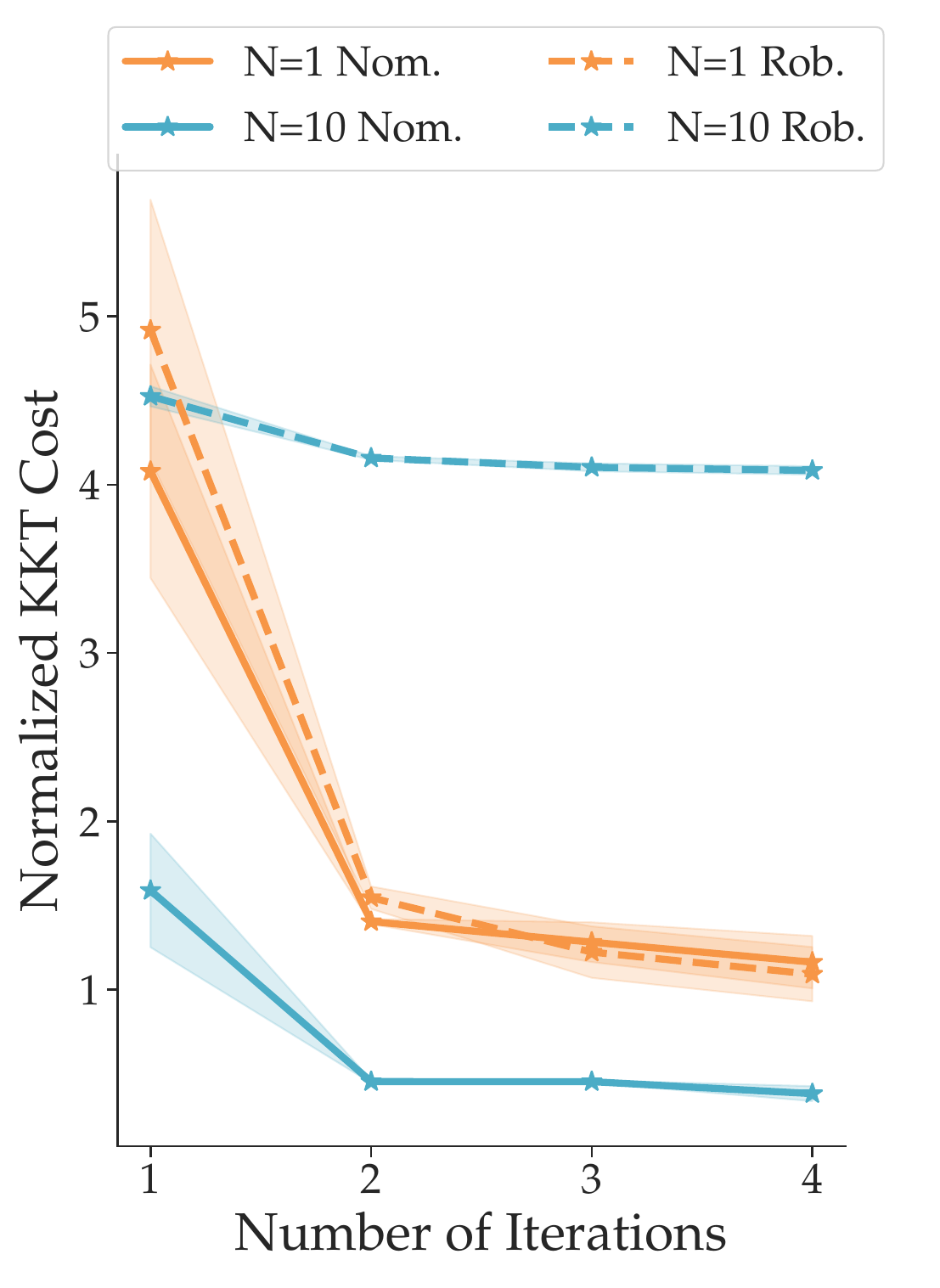}
\caption{}
\label{fig:conv_rate_3d}
\end{subfigure}
\caption{KKT residual cost function, normalized by the number of trajectories, with respect to iterations for the 2D navigation task~(Fig.~\ref{fig:conv_rate_2d}) and 3D manipulation task~(Fig.~\ref{fig:conv_rate_3d}). Results averaged over $10$ simulations.} 
\label{fig:conv_rates}
\end{figure}



\section{Conclusion and Future Extensions}\label{sec:conclusion}
In this work, we presented a method that infers constraints in control tasks after observing expert demonstrations. We showed conditions under which exact recovery is possible. We proposed to learn the unknown constraints by minimizing the KKT residual objective function. We further introduced IGCI, an algorithm that can be utilized to infer constraints in cases where no information, except for their parametric representation, is available. 
We evaluated the performance of IGCI in a number of simulations, in which we showed that the inferred constraints can be accurate without jeopardizing safety under perfect, noisy and suboptimal observations. 

The alternating minimization procedure presented in this paper is flexible enough to handle more complicated constraints that can be modeled as unions and intersections of half-spaces. {Another promising direction is designing control inputs to actively gather information about unknown convex or even non-convex constraints in potentially nonlinear control problems. } 

\bibliographystyle{IEEEtran}
\bibliography{IEEEabrv,references}

\begin{thebibliography}{10}
\providecommand{\url}[1]{#1}
\csname url@samestyle\endcsname
\providecommand{\newblock}{\relax}
\providecommand{\bibinfo}[2]{#2}
\providecommand{\BIBentrySTDinterwordspacing}{\spaceskip=0pt\relax}
\providecommand{\BIBentryALTinterwordstretchfactor}{4}
\providecommand{\BIBentryALTinterwordspacing}{\spaceskip=\fontdimen2\font plus
\BIBentryALTinterwordstretchfactor\fontdimen3\font minus
  \fontdimen4\font\relax}
\providecommand{\BIBforeignlanguage}[2]{{%
\expandafter\ifx\csname l@#1\endcsname\relax
\typeout{** WARNING: IEEEtran.bst: No hyphenation pattern has been}%
\typeout{** loaded for the language `#1'. Using the pattern for}%
\typeout{** the default language instead.}%
\else
\language=\csname l@#1\endcsname
\fi
#2}}
\providecommand{\BIBdecl}{\relax}
\BIBdecl

\bibitem{ng2000algorithms}
A.~Y. Ng, S.~Russell \emph{et~al.}, ``Algorithms for inverse reinforcement
  learning.'' in \emph{Icml}, vol.~1, 2000, p.~2.

\bibitem{hussein2017imitation}
A.~Hussein, M.~M. Gaber, E.~Elyan, and C.~Jayne, ``Imitation learning: A survey
  of learning methods,'' \emph{ACM Computing Surveys (CSUR)}, vol.~50, no.~2,
  pp. 1--35, 2017.

\bibitem{liu2022benchmarking}
G.~Liu, Y.~Luo, A.~Gaurav, K.~Rezaee, and P.~Poupart, ``Benchmarking constraint
  inference in inverse reinforcement learning,'' \emph{arXiv preprint
  arXiv:2206.09670}, 2022.

\bibitem{chou2020learning}
G.~Chou, N.~Ozay, and D.~Berenson, ``Learning constraints from locally-optimal
  demonstrations under cost function uncertainty,'' \emph{IEEE Robotics and
  Automation Letters}, vol.~5, no.~2, pp. 3682--3690, 2020.

\bibitem{chan2021inverse}
T.~C. Chan, R.~Mahmood, and I.~Y. Zhu, ``Inverse optimization: Theory and
  applications,'' \emph{arXiv preprint arXiv:2109.03920}, 2021.

\bibitem{scobee2019maximum}
D.~R. Scobee and S.~S. Sastry, ``Maximum likelihood constraint inference for
  inverse reinforcement learning,'' \emph{arXiv preprint arXiv:1909.05477},
  2019.

\bibitem{malik2021inverse}
S.~Malik, U.~Anwar, A.~Aghasi, and A.~Ahmed, ``Inverse constrained
  reinforcement learning,'' in \emph{International Conference on Machine
  Learning}.\hskip 1em plus 0.5em minus 0.4em\relax PMLR, 2021, pp. 7390--7399.

\bibitem{papadimitrioubayesian}
D.~Papadimitriou, U.~Anwar, and D.~S. Brown, ``Bayesian methods for constraint
  inference in reinforcement learning,'' \emph{Transactions on Machine Learning
  Research}.

\bibitem{perez2017c}
C.~P{\'e}rez-D'Arpino and J.~A. Shah, ``C-learn: Learning geometric constraints
  from demonstrations for multi-step manipulation in shared autonomy,'' in
  \emph{2017 IEEE International Conference on Robotics and Automation
  (ICRA)}.\hskip 1em plus 0.5em minus 0.4em\relax IEEE, 2017, pp. 4058--4065.

\bibitem{chou2018learning}
G.~Chou, D.~Berenson, and N.~Ozay, ``Learning constraints from
  demonstrations,'' in \emph{International Workshop on the Algorithmic
  Foundations of Robotics}.\hskip 1em plus 0.5em minus 0.4em\relax Springer,
  2018, pp. 228--245.

\bibitem{robey2020learning}
A.~Robey, H.~Hu, L.~Lindemann, H.~Zhang, D.~V. Dimarogonas, S.~Tu, and
  N.~Matni, ``Learning control barrier functions from expert demonstrations,''
  in \emph{2020 59th IEEE Conference on Decision and Control (CDC)}.\hskip 1em
  plus 0.5em minus 0.4em\relax IEEE, 2020, pp. 3717--3724.

\bibitem{molloy2020online}
T.~L. Molloy, J.~J. Ford, and T.~Perez, ``Online inverse optimal control for
  control-constrained discrete-time systems on finite and infinite horizons,''
  \emph{Automatica}, vol. 120, p. 109109, 2020.

\bibitem{agrawal2021learning}
A.~Agrawal, S.~Barratt, and S.~Boyd, ``Learning convex optimization models,''
  \emph{IEEE/CAA Journal of Automatica Sinica}, vol.~8, no.~8, pp. 1355--1364,
  2021.

\bibitem{menner2019constrained}
M.~Menner, P.~Worsnop, and M.~N. Zeilinger, ``Constrained inverse optimal
  control with application to a human manipulation task,'' \emph{IEEE
  Transactions on Control Systems Technology}, vol.~29, no.~2, pp. 826--834,
  2019.

\bibitem{li2023cost}
J.~Li, C.-Y. Chiu, L.~Peters, S.~Sojoudi, C.~Tomlin, and D.~Fridovich-Keil,
  ``Cost inference for feedback dynamic games from noisy partial state
  observations and incomplete trajectories,'' \emph{arXiv preprint
  arXiv:2301.01398}, 2023.

\bibitem{ahuja2001inverse}
R.~K. Ahuja and J.~B. Orlin, ``Inverse optimization,'' \emph{Operations
  Research}, vol.~49, no.~5, pp. 771--783, 2001.

\bibitem{chan2020inverse}
T.~C. Chan and N.~Kaw, ``Inverse optimization for the recovery of constraint
  parameters,'' \emph{European Journal of Operational Research}, vol. 282,
  no.~2, pp. 415--427, 2020.

\bibitem{ghobadi2021inferring}
K.~Ghobadi and H.~Mahmoudzadeh, ``Inferring linear feasible regions using
  inverse optimization,'' \emph{European Journal of Operational Research}, vol.
  290, no.~3, pp. 829--843, 2021.

\bibitem{englert2017inverse}
P.~Englert, N.~A. Vien, and M.~Toussaint, ``Inverse kkt: Learning cost
  functions of manipulation tasks from demonstrations,'' \emph{The
  International Journal of Robotics Research}, vol.~36, no. 13-14, pp.
  1474--1488, 2017.

\bibitem{awasthi2019forward}
C.~Awasthi, ``Forward and inverse methods in optimal control and dynamic game
  theory,'' Ph.D. dissertation, University of Minnesota, 2019.

\bibitem{menner2020maximum}
M.~Menner and M.~N. Zeilinger, ``Maximum likelihood methods for inverse
  learning of optimal controllers,'' \emph{IFAC-PapersOnLine}, vol.~53, no.~2,
  pp. 5266--5272, 2020.

\bibitem{keshavarz2011imputing}
A.~Keshavarz, Y.~Wang, and S.~Boyd, ``Imputing a convex objective function,''
  in \emph{2011 IEEE international symposium on intelligent control}.\hskip 1em
  plus 0.5em minus 0.4em\relax IEEE, 2011, pp. 613--619.

\bibitem{gorski2007biconvex}
J.~Gorski, F.~Pfeuffer, and K.~Klamroth, ``Biconvex sets and optimization with
  biconvex functions: a survey and extensions,'' \emph{Mathematical methods of
  operations research}, vol.~66, no.~3, pp. 373--407, 2007.

\bibitem{bertsekas1999nonlinear}
D.~P. Bertsekas, \emph{Nonlinear programming}.\hskip 1em plus 0.5em minus
  0.4em\relax Athena Scientific, 1999.

\bibitem{gymnasium_robotics2023github}
\BIBentryALTinterwordspacing
R.~de~Lazcano, K.~Andreas, J.~J. Tai, S.~R. Lee, and J.~Terry, ``Gymnasium
  robotics,'' 2023. [Online]. Available:
  \url{http://github.com/Farama-Foundation/Gymnasium-Robotics}
\BIBentrySTDinterwordspacing

\end{thebibliography}
\appendix

\subsection{Simulation Details}\label{sec:sim_detail}
Although IGCI is not very sensitive in the choice of $\rho_1,\rho_2$, in all simulations we use $\rho_1=10$ and $\rho_2=0$. \textit{2D Navigation Task}:
The tracking trajectory is $x_{track}=[0.,-5.5;-0.6,-4.5;-1.5,-3.8;-1.5,-3.5;-1.6,-2.5;\\-0.2,-2;0.5,-2;0.8,-1.5]^{\top}$. 

\noindent\textit{Gym Environment}:
The tracking trajectory is $x_{track}=[0.95,1,0.8;0.9,1,0.8;0.85,1,0.8;0.8,1.1,0.9;0.9,1.15,\\0.9;0.95,1.2,0.9;1,1.25,0.9;1.1,1.3,0.9]^{\top}$.

\end{document}